\documentclass{article} % For LaTeX2e
\usepackage{times}
\usepackage[accepted]{icml2014}
\usepackage{natbib} 
\usepackage{url}

\usepackage{graphicx}
\usepackage{amsfonts}
\usepackage{amsthm}
\usepackage{amssymb,amsmath}
\usepackage{subfigure}

\newcommand{\R}{\mathbb{R}}

\newcommand{\Esp}{\mathbb{E}}

\newtheorem{theorem}{Theorem}
\newtheorem{proposition}{Proposition}
\newtheorem{lemma}{Lemma}

\newboolean{vl}
\setboolean{vl}{true}

\begin{document}

\twocolumn[
\icmltitle{Dimension-free Concentration Bounds on Hankel Matrices for
  Spectral Learning}

\icmlauthor{Fran\c{c}ois Denis}{francois.denis@lif.univ-mrs.fr}
\icmlauthor{Mattias Gybels}{mattias.gybels@lif.univ-mrs.fr}
\icmladdress{Aix Marseille Universit\'e, CNRS, LIF, 13288 Marseille
  Cedex 9, FRANCE}
\icmlauthor{Amaury Habrard}{amaury.habrard@univ-st-etienne.fr}
\icmladdress{Universit\'e Jean Monnet de Saint-Etienne, CNRS, LaHC, 42000 Saint-Etienne
  Cedex 2, FRANCE}
\icmlkeywords{Spectral Learning}

\vskip 0.3in
]

\begin{abstract}
Learning probabilistic models over strings is an important issue for many
applications. Spectral methods propose
elegant solutions to the problem of inferring weighted
automata from finite samples of variable-length strings drawn from an
unknown target distribution. These methods rely on a singular
value decomposition of a matrix $H_S$, called the Hankel matrix, that
records the frequencies of (some of) the observed 
strings. The accuracy of the learned distribution depends both on the
quantity of information embedded in $H_S$ and on the distance between
$H_S$ and its mean $H_r$. Existing concentration bounds
seem to indicate that the concentration over $H_r$ gets looser with
the size of
$H_r$, suggesting to make a trade-off between the quantity of used information
and the size of $H_r$. We propose new dimension-free concentration bounds for several
variants of Hankel matrices. Experiments demonstrate that these bounds are tight and that they significantly improve
existing bounds. These results suggest that the concentration rate of the
Hankel matrix around its mean does not constitute an argument for limiting
its size.
\end{abstract}
%%% Local Variables: 
%%% mode: latex
%%% TeX-master: "ICML2014"
%%% End: 

\section{Introduction}
Many applications in natural language processing, text analysis or computational biology require learning  probabilistic models over finite variable-size strings such as probabilistic automata, Hidden Markov Models (HMM), or more generally, weighted automata. Weighted automata exactly model the class of rational series, and their algebraic properties have been widely studied in that context~\cite{Droste:2009}. In particular, they admit algebraic representations that can be characterized by a set of finite-dimensional linear operators whose rank corresponds to the minimum number of states needed to define the automaton. From a machine learning perspective, the objective is then to infer good estimates of these linear operators from finite samples. In this paper, we consider the problem of learning the linear representation of a weighted automaton, from a finite sample, composed of variable-size strings i.i.d. from an unknown target distribution.

Recently, the seminal papers of \citet{HsuKZ09} for learning HMM and  \citet{BaillyDR09} for  weighted automata, have defined a new category of 
approaches - the so-called \emph{spectral methods} - for learning distributions over strings represented by finite state models \cite{Siddiqi10a,DBLP:conf/icml/SongSGS10,DBLP:conf/icml/BalleQC12,DBLP:conf/nips/BalleM12}. Extensions to probabilistic models for
tree-structured  data~\cite{DBLP:conf/alt/BaillyHD10,DBLP:conf/icml/ParikhSX11,Cohen2012}, 
 transductions \cite{DBLP:conf/pkdd/BalleQC11} or other graphical
models \cite{AnandkumarHK12,AnandkumarHHK12,AnandkumarFHKL12,LuqueQBC12} have also attracted a lot of interest.  

Spectral methods suppose that the main parameters of a model can be expressed as the spectrum of a linear operator and estimated from the spectral decomposition of a matrix that sums up the observations.  Given a rational series $r$, the values taken by $r$ can be arranged in a matrix $H_r$ whose rows and columns are indexed by strings, such that the linear operators defining $r$ can be recovered directly from the right singular vectors of $H_r$. This matrix is called the Hankel matrix of $r$. 

In a learning context, given a learning sample $S$ drawn from a target distribution $p$, an empirical estimate $H_S$ of $H_p$ is built and then, a rational series $\tilde{p}$ is inferred from the right singular vectors of $H_S$. However, the size of $H_S$ increases drastically with the size of $S$ and state of the art approaches consider smaller matrices $H_S^{U,V}$ indexed by limited subset of strings $U$ and $V$. It can be shown that the above learning scheme, or slight variants of it, are consistent as soon as the matrix $H_S^{U,V}$ has full rank~\cite{HsuKZ09,Bailly2011,DBLP:conf/icml/BalleQC12} and that the accuracy of the inferred series is directly connected to the concentration distance $||H_S^{U,V}-H_p^{U,V}||_2$ between the empirical Hankel matrix and its mean~\cite{HsuKZ09,Bailly2011}. 

On the one hand, limiting the size of the Hankel matrix avoids prohibitive calculations. Moreover, most existing concentration bounds on sum of random matrices depend on their size and suggest that $||H_S^{U,V}-H_p^{U,V}||_2$ may become significantly looser with the size of $U$ and $V$, compromising the accuracy of the inferred model. 

On the other hand,  limiting the size of the Hankel matrix implies a drastic loss of information: only the strings of $S$ compatible with $U$ and $V$  will be considered. In order to limit the loss of information when dealing with restricted sets $U$ and $V$, a general trend is to work with other functions than the target $p$, such as the \emph{prefix} function $\overline{p}(u)=\sum_{v\in \Sigma^*}p(uv)$ or the \emph{factor} function $\widehat{p}=\sum_{v,w\in \Sigma^*}p(vuw)$~\cite{Balle2013,LuqueQBC12}. These functions are rational, they have the same rank as $p$, a representation of $p$ can easily be derived from representations of $\overline{p}$ or $\widehat{p}$ and they allow a better use of the information contained in the learning sample. 

A first contribution is to provide a dimension free concentration inequality for $||H_S^{U,V}-H_p^{U,V}||_2$, by using recent results on tail inequalities for sum of random matrices showing that restricting the dimension of $H$ is not mandatory.

However, these results cannot be directly applied as such to the prefix and factor series, since the norm of the corresponding random matrices are unbounded.  A second contribution of the paper is then to define two classes of parametrized functions, $\overline{p}_{\eta}$ and $\widehat{p}_{\eta}$,  that constitute continuous intermediates between $p$ and $\overline{p}$ (resp. $p$ and $\widehat{p}$), and to provide analogous dimension-free concentration bounds for these two classes.

These bounds are evaluated on a benchmark made of 11 problems extracted from the PAutomaC challenge~\cite{VerwerEH12}. These experiments show that the bounds derived from our theoretical results are quite tight - compared to the exact values-  and that they significantly improve existing bounds, even on matrices of fixed dimensions. 

These results have two practical consequences for spectral learning: (i) the concentration of the empirical Hankel matrix around its mean does not highly depend on its dimension and the only reason not to use all the information contained in the sample should only rely on computing resources limitations. In that perspective, using random techniques to perform singular values decomposition on huge Hankel matrices should be considered~\cite{Halko:2011}; (ii) by constrast, the concentration is weaker for the prefix and factor functions, and smoothed variants should be used, with an appropriate parameter.

The paper is organized as follows. Section~\ref{s:preliminaries}
introduces the main notations, definitions and concepts.  Section~\ref{s:bound}  presents a first dimension free-concentration inequality for the standard Hankel matrices. Then, we introduce the prefix and the factor variants and provide analogous concentration results. 
Section~\ref{s:expe} describes some experiments before the conclusion presented in Section~\ref{s:conclu}.

%%% Local Variables: 
%%% mode: latex
%%% TeX-master: "ICML2014"
%%% End: 

\section{Preliminaries}
\label{s:preliminaries}
\vspace{-1mm}\subsection{Singular Values, Eigenvalues and Matrix Norms}
Let $M\in \R^{m\times n}$ be a $m\times 
n$ real matrix. The singular values of $M$ are the square roots of the eigenvalues
of the  matrix $M^TM$, where $M^T$ denotes the transpose of $M$: $\sigma_{max}(M)$ and $\sigma_{min}(M)$
denote the largest and smallest singular value of $M$, respectively.

In this paper, we mainly use the spectral norms $||\cdot||_k$  induced by the corresponding vector
  norms on $\R^n$ and defined by $||M||_k=max_{x\neq
    0}\frac{||Mx||_k}{||x||_k}$:
  \begin{itemize}
  \item $||M||_1=Max_{1\leq j\leq n}\sum_{i=1}^m|M[i,j]|$,
  \item $||M||_{\infty}=Max_{1\leq i\leq m}\sum_{j=1}^n|M[i,j]|$,
  \item $||M||_2=\sigma_{max}(M)$.
  \end{itemize}

We have: $||M||_2 \leq \sqrt{||M||_1||M||_{\infty}}$.

These norms can be extended, under certain conditions, to infinite matrices and the previous inequalities remain true when the corresponding norms are
defined. 

\vspace{-1mm}\subsection{Rational stochastic languages and Hankel matrices}

Let $\Sigma$ be a finite alphabet. The set of all finite strings over
$\Sigma$ is denoted by $\Sigma^*$, the empty string is denoted by
$\epsilon$, the length of string $w$ is denoted
by $|w|$ and $\Sigma^n$ (resp. $\Sigma^{\leq n}$) denotes the set of
all strings of length $n$ (resp. $\leq n$). 
For any string $w$, let  $\mathrm{Pref}(w)\!=\!\{u\in \Sigma^*|\exists
v\in \Sigma^*\ w=uv\}$.

A \emph{series} is a mapping $r:\Sigma^* \mapsto \R$. A series
$r$ is convergent if the sequence $r(\Sigma^{\leq
  n})=\sum_{w\in\Sigma^{\leq n}}r(w)$ is convergent; its
limit is denoted by $r(\Sigma^*)$. A \emph{stochastic language} $p$ is a probability
distribution over $\Sigma^*$, i.e.\ a series taking non negative
values and converging to 1.

Let $n\geq 1$ and $M$ be a morphism defined from $\Sigma^*$ to ${\mathcal
  M}(n)$, the set of $n\times n$ matrices with real
coefficients. For all $u\in \Sigma^*$, let us denote $M(u)$ by $M_u$
and $\Sigma_{x\in \Sigma}M_x$ by $M_{\Sigma}$.  A series $r$ over
$\Sigma$ is \emph{rational} if there exists an integer $n\geq 1$, two vectors
$I,T\in \R^n$ and a morphism $M: \Sigma^*\mapsto {\mathcal
  M}(n)$ such that for all $u\in
\Sigma^*$, $r(u)=I^TM_uT$.
The triplet  $\langle I,M,T\rangle$  is called an $n$-dimensional \emph{linear
  representation} of $r$. The vector $I$ can be interpreted as a
vector of initial weights, $T$ as a vector of terminal weights and the
morphism $M$ as a set of matrix parameters associated with the letters
of $\Sigma$. A \emph{rational stochastic language} is thus a
stochastic language admitting a linear representation.

Let $U,V\subseteq \Sigma^*$, the \emph{Hankel matrix} $H_r^{U,V}$,
associated with a series $r$, is the matrix indexed by $U\times V$ and
defined by $H_r^{U,V}[u,v]=r(uv)$, for any $(u,v)\in U\times V$. If
$U=V=\Sigma^*$, $H_r^{U,V}$, simply denoted by $H_r$, is a bi-infinite
matrix. In the following, we always assume that $\epsilon\in U$ and
that $U$ and $V$ are ordered in quasi-lexicographic order: strings are
first ordered by increasing length and then, according to the
lexicographic order.  It can be shown that a series $r$ is rational if
and only if the rank of the matrix $H_r$ is finite. The rank of $H_r$
is equal to the minimal dimension of a linear representation of $r$.

Let $r$ be a non negative convergent rational series and let $\langle
I,M,T\rangle$ be a minimal $d$-dimensional linear representation of
$r$.  Then, the sum $I_d+M_{\Sigma}+\ldots +M_{\Sigma}^n+ \ldots$
is convergent and $r(\Sigma^*)=I^T(I_d-M_{\Sigma})^{-1}T$ where $I_d$ is
the identity matrix of size $d$.

Several convergent rational series can be naturally associated with a stochastic
language $p$:
\begin{itemize}
\item $\overline{p}$, defined by $\overline{p}(u)=\sum_{v\in \Sigma^*}p(uv)$, the
  series associated with the \emph{prefixes} of the  language, 
\item $\widehat{p}$, defined by $\hat{p}(u)=\sum_{v,w\in \Sigma^*}p(vuw)$,
 the series associated with the  \textit{factors} of the language. 
\end{itemize}
It can be noticed that $\overline{p}(u)=p(u\Sigma^*)$, the probability
that a string begins with $u$, but that in general, $\widehat{p}(u)\geq
p(\Sigma^*u\Sigma^*)$, the probability
that a string contains $u$ as a substring. 

If $\langle I,M,T\rangle$ is a minimal $d$-dimensional linear
representation of $p$, then $\langle
I,M,(I_d-M_{\Sigma})^{-1}T\rangle$ (resp. $\langle
[I^T(I_d-M_{\Sigma})^{-1}]^T,M,(I_d-M_{\Sigma})^{-1}T\rangle$) is a minimal linear
representation of $\overline{p}$ (resp. of $\widehat{p}$).
Any linear representation of these variants of $p$ can be 
reconstructed from the others.

For any
integer $k\geq 1$, let $$S_p^{(k)}=\sum_{u_1u_2\ldots u_k\in
  \Sigma^*}p(u_1u_2\ldots u_k) =I^T(I_d-M_{\Sigma})^{-k}T.$$ 

Clearly, $p(\Sigma^*)$\hspace{0.5mm}$=$\hspace{0.5mm}$S_p^{(1)}$\hspace{0.5mm}$=$\hspace{0.5mm}$1$,
$\overline{p}(\Sigma^*)$\hspace{0.5mm}$=$\hspace{0.5mm}$S_p^{(2)}$ and
$\widehat{p}(\Sigma^*)$\hspace{0.5mm}$=$\hspace{0.5mm}$S_p^{(3)}.$

Let $U,V\subseteq \Sigma^*$. For any string $w\in \Sigma^*$, let us
define the matrices $H_w^{U,V}$,  $\overline{H}_w^{U,V}$ and $\widehat{H}_w^{U,V}$ by
\begin{itemize}
\item $H_w^{U,V}[u,v]={\mathbf 1}_{uv=w}$,
\item $\overline{H}_w^{U,V}[u,v]={\mathbf 1}_{uv\in Pref(w)}$ and
\item $\widehat{H}_w^{U,V}[u,v]=\sum_{x,y\in \Sigma^*}{\mathbf 1}_{xuvy=w}$
\end{itemize}
 for any $(u,v)\in U\times
  V$.
For any sample of strings $S$, let
$H_S^{U,V}=\frac{1}{|S|}\sum_{w\in S}H_w^{U,V}$, $\overline{H}_S^{U,V}=\frac{1}{|S|}\sum_{w\in
  S}\overline{H}_w^{U,V}$ and
$\widehat{H}_S^{U,V}=\frac{1}{|S|}\sum_{w\in
  S}\widehat{H}_w^{U,V}$. 

For example, let $S=\{a,ab\}$,
$U=V=\{\epsilon,a,b\}$. We have \tiny
$$\hspace{-1mm}H^{U,V}_S\hspace{-1.2mm}=\begin{pmatrix}
\hspace{-0.1mm}0\hspace{-1.2mm}&\hspace{-1.2mm}0.5\hspace{-1.2mm}&\hspace{-1.2mm}0\\
\hspace{-0.1mm}0.5\hspace{-1.2mm}&\hspace{-1.2mm}0\hspace{-1.2mm}&\hspace{-1.2mm}0.5\\
\hspace{-0.1mm}0\hspace{-1.2mm}&\hspace{-1.2mm}0\hspace{-1.2mm}&\hspace{-1.2mm}0
\end{pmatrix},\hspace{2mm}
\overline{H}^{U,V}_S\hspace{-1.2mm}=\begin{pmatrix}
\hspace{0.7mm}1\hspace{-0.8mm}&\hspace{-0.8mm}1\hspace{-1.2mm}&\hspace{-1.4mm}0\\
\hspace{0.7mm}1\hspace{-0.8mm}&\hspace{-0.8mm}0\hspace{-1.2mm}&\hspace{-1.4mm}0.5\\
\hspace{0.7mm}0\hspace{-0.8mm}&\hspace{-0.8mm}0\hspace{-1.2mm}&\hspace{-1.4mm}0
\end{pmatrix},\hspace{2mm}
\widehat{H}^{U,V}_S\hspace{-1.2mm}=\begin{pmatrix}
\hspace{0.3mm}2.5\hspace{-1.2mm}&\hspace{-1.2mm}1\hspace{-1.2mm}&\hspace{-1.2mm}0.5\\
\hspace{0.3mm}1\hspace{-1.2mm}&\hspace{-1.2mm}0\hspace{-1.2mm}&\hspace{-1.2mm}0.5\\
\hspace{0.3mm}0.5\hspace{-1.2mm}&\hspace{-1.2mm}0\hspace{-1.2mm}&\hspace{-1.2mm}0
\end{pmatrix}.
$$
\normalsize

\vspace{-1mm}\subsection{Spectral Algorithm for Learning Rational Stochastic Languages}
\label{inference}
Rational series admit a \emph{canonical linear representation} determined by
their Hankel matrix.  Let $r$ be a rational series of rank $d$ and
$U\subset \Sigma^*$ such that the matrix $H^{U\times \Sigma^*}_r$
(denoted by $H$ in the following) has rank $d$.
\begin{itemize}
\item For any string $s$, let $T_s$ be the constant matrix whose
rows and columns are indexed by $\Sigma^*$ and defined by
$T_s[u,v]=1$ if $v=us$ and 0 otherwise.
\item Let $E$ be a vector  indexed by $\Sigma^*$ whose coordinates are
  all zero except the first one equals to 1:
  $E[u]={\mathbf 1}_{u=\epsilon}$ and let $P$ be the vector indexed by $\Sigma^*$ defined by $P[u]=r(u)$.
\item Let $H=LDR^T$ be a reduced singular value
decomposition of $H$: $R$ (resp. $L$) is a matrix 
whose columns form a set of orthonormal vectors - the right
(resp. left) singular
vectors of $H$ - and $D$ is a $d\times d$ diagonal matrix, composed of
the singular values of $H$. 
\end{itemize}
Then, $\langle R^TE,(R^TT_{x}R)_{x\in
  \Sigma},R^TP\rangle$ is a linear representation of
$r$~\cite{BaillyDR09,HsuKZ09,Bailly2011,DBLP:conf/icml/BalleQC12}.  
\ifthenelse{\boolean{vl}}{
\begin{proposition}
$\langle R^TE,(R^TT_{x}R)_{x\in
  \Sigma},R^TP\rangle$ is a linear representation of
$r$
\end{proposition}
\begin{proof}

From the definition of $T_s$, it can easily be shown that the mapping  $s\mapsto T_s$ is a morphism:\\ $T_{s_1}T_{s_2}[u,v]=\sum_{w\in
  \Sigma^*}T_{s_1}[u,w]T_{s_2}[w,v]=1$ iff $v=us_1s_2$ and 0
otherwise.  \\
If $X$ is a matrix whose rows are indexed by $\Sigma^*$, we have
$T_sX[u,v]=\sum_wT_s[u,w]X[w,v]=X[us,v]$: ie the rows of $T_SX$ are
included in the set of rows of $X$.
Then, it follows from the definition of $E$ that $E^TT_s$ is equal to the first row of $T_s$ (indexed by $\epsilon$) with  all coordinates
equal to zero except the one indexed by $s$ which equal 1. 

Now, from the reduced singular value
decomposition of $H=LDR^T$ at rank $d$, $R$ is a matrix of dimension $\infty\times d$
whose columns form a set of orthonormal vectors - the right singular
vectors of $H$ - such that $R^TR\!=\!I_d$ and $RR^TH^T\!=\!H^T$ ($R R^T$ is the
orthogonal projection on the subspace spanned by the rows of $H$). \\
One can easily deduce, by a recurrence over $n$, that for every string $u=x_1\ldots x_n$, \\$(R^TT_{x_1}R)\circ \ldots \circ
(R^TT_{x_n}R)R^TH^T=R^TT_u H^T.$\\
Indeed, the inequality is trivially true for $n=0$ since
$T_{\epsilon}=I_d$. Then, we have that
$R^TT_xRR^TT_uH^T=R^TT_xT_uH^T=R^TT_{xu}H^T$ since the columns of $T_uH^T$
are rows of $H$ and $T$ is a morphism.

If $P^T$ is the first row of $H$ then:\\   
$E^T\!R(\!R^TT_{x_1}R)\!\circ\ldots \circ\!  (R^TT_{x_n}R)R^TP\!=\!E^TT_u
P\!=\!r(u)$. Thus,  $\langle R^TE,(R^TT_{x}R)_{x\in
  \Sigma},R^TP\rangle$ is a linear representation of $r$ of
dimension $d$. Note here that $r$ is only needed in the right
singular vectors $R$ and  in the vector $P$. 
\end{proof}
}
{A quick proof can be found in~\cite{Denis2013}.}

The basic spectral algorithm for learning rational
stochastic languages aims at identifying the canonical linear representation of
the target $p$ determined by its Hankel matrix $H_p$. 

Let $S$ be a sample independently drawn according to $p$:
\begin{itemize}
\item  Choose sets  $U,V\subseteq \Sigma^*$ and build the Hankel matrix $H^{U\times V}_S$,
\item  choose a rank $d$ and compute a reduced SVD of $H^{U\times V}_S$ truncated at rank $d$,
\item  build the canonical linear representation $\langle R_S^TE,(R_S^TT_{x}R_S)_{x\in
    \Sigma},R_S^TP_S\rangle$ from the right singular vectors $R_S$
  and the empirical distribution $p_S$ defined from $S$. 
\end{itemize} 

Alternative learning strategies consist in learning $\overline{p}$ or
$\widehat{p}$, using the same algorithm, and then to compute an
estimate of $p$. In all cases, the accuracy
of the learned representation mainly depends on the estimation of
$R$. The Stewart formula~\cite{Stewart90perturbationtheory} bounds the
principle angle $\theta $ between the spaces spanned by the right
singular vectors of $R$ and $R_S$:
$$|sin(\theta)|\leq \frac{||H^{U\times V}_S-H^{U\times
    V}_r||_2}{\sigma_{min}(H^{U\times V}_r)}.$$ According to this
formula, the concentration of the Hankel matrix around its mean is
critical and the question of limiting the sizes of $U$ and $V$
naturally arises. Note that the Stewart inequality does not give any
clear indication on the impact or on the interest of limiting these
sets. Indeed, Weyl's inequalities can be used to show that both the
numerator and the denominator of the right part of the inequality
increase with $U$ and $V$.

%%% Local Variables: 
%%% mode: latex
%%% TeX-master: "ICML2014"
%%% End: 

\section{Concentration Bounds for Hankel Matrices}
\label{s:bound}
Let $p$ be a rational stochastic language over $\Sigma^*$, let $\xi$ be a random variable distributed according
to $p$, let $U, V\subseteq \Sigma^*$ and let $Z(\xi)\in
\R^{|U|\times|V|}$ be a random matrix. For instance, $Z(\xi)$ may be
equal to $H_{\xi}^{U,V}$,  $\overline{H}_{\xi}^{U,V}$ or
$\widehat{H}_{\xi}^{U,V}$.

Concentration bounds for sum of random matrices can be used to
estimate the spectral distance between the empirical matrix $Z_S$
computed on the sample $S$ and its mean
(see~\cite{2011arXiv1104.1672H} for references). However, most of
classical inequalities depend on the dimensions of the matrices. For
example, it can be proved that with
probability at least $1-\delta$ \cite{Kakade2010}:
\begin{equation}
||Z_S-\Esp Z||_2\leq
\frac{6M}{\sqrt{N}}\left(\sqrt{\log{d}}+\sqrt{\log{\frac{1}{\delta}}}\right)\label{eq:1}
\end{equation}
where $N$ is the size of $S$, $d$ is the  minimal dimension of the
matrix $Z$ and $||Z||_2\leq M$ almost surely. If
$Z=H_{\xi}^{U,V}$, then $M=1$;   if $Z=\overline{H}_{\xi}^{U,V}$,
$M=\Omega(D^{1/2})$ in the worst case; if $Z=\widehat{H}_{\xi}^{U,V}$,
$||Z||_2$ is generally unbounded. 

These concentration bounds get worse with both sizes of the
matrices. Coming back to the discussion at the end of
Section~\ref{s:preliminaries}, they suggest to limit the size of the sets $U$ and $V$, and
therefore, to design strategies to choose optimal sets.  

We then use recent results
\cite{DBLP:journals/focm/Tropp12,2011arXiv1104.1672H} to
obtain dimension-free concentration bounds for Hankel matrices.
\ifthenelse{\boolean{vl}}{Let $\xi_1, \ldots, \xi_N$ be some random variables and for each $i=1,
\ldots, N$, and let $X_i=X_i(\xi_1, \ldots, \xi_i)$ be a  \textbf{random}
matrix function of $\xi_1, \ldots, \xi_i$. The notation
$\Esp_i[\cdot]$ is a shortcut for $\Esp[\cdot|  \xi_1, \ldots,
\xi_{i-1}]$. 

\begin{theorem} (Matrix Bernstein Bound)\cite{2011arXiv1104.1672H}. \label{th1}
If there exists $b\!>\!0, \sigma\!>\!0, k\!>\!0$ s.t. for all $i=1,
\ldots, N$, 
$\Esp_i[X_i]=0,\ ||X_i||_2\leq b,\  
||\frac{1}{N}\sum_{i=1}^N \Esp_i (X_i^2)||_2\leq \sigma^2
\mbox{ and } \Esp\left[tr\left(\frac{1}{N}\sum_{i=1}^N \Esp_i
  (X_i^2)\right)\right]\leq \sigma^2k$
almost surely, then for all $t>0$,  

 $\displaystyle Pr\left[\lambda_{max}\left(\frac{1}{N}\sum_{i=1}^N
    X_i\right)\!>\!\sqrt{\frac{2\sigma^2t}{N}}+\frac{bt}{3N}\right]\!\leq\!\frac{k\cdot t}{e^t-t-1}.$
\end{theorem}

We use this theorem in the particular case where the random variables
$\xi_i$ are i.i.d. and each matrix $X_i$ depends only on $\xi_i$.  

This theorem is valid for symmetric matrices, but it can be extended to general real-valued matrices thanks to the principle of dilation.

}
{More precisely, we extend a Bernstein bound  for unbounded random matrices
from~\cite{2011arXiv1104.1672H} to non symmetric
random matrices by using the dilation principle \cite{DBLP:journals/focm/Tropp12}.} 

Let $Z$ be a random matrix, the \emph{dilation}
 of $Z$ is the symmetric random matrix
$X$ defined by
 $$X=
\left[\begin{array}{cc}
  0 &Z\\Z^T& 0
\end{array}\right]. \textrm{\ Then }
X^2=\left[\begin{array}{cc}
  ZZ^T &0\\0&Z^TZ
\end{array}\right]$$ 
and $||X||_2=||Z||_2$, $tr(X^2)=tr(ZZ^T)+tr(Z^TZ)$ and $||X^2||_2\leq Max(||ZZ^T||_2,||Z^TZ||_2)$.

We can then reformulate the result that we use as follows\ifthenelse{\boolean{vl}}{}{~\cite{Denis2013}}.
\begin{theorem}\label{conc2}
Let $\xi_1, \ldots, \xi_N$ be  i.i.d. random variables, and for $i=1,
\ldots, N$, let 
$Z_i=Z(\xi_i)$ be i.i.d. matrices and $X_i$ the dilation of $Z_i$. If there exists $b>0,
\sigma>0$, and $k>0$ such that 
$\Esp[X_1]=0,\ ||X_1||_2\leq b, ||\Esp (X_1^2)||_2\leq \sigma^2 \mbox{ and } tr(\Esp
  (X_1^2))\leq \sigma^2k$
almost surely, then for all $t>0$, 
$$Pr\left[||\frac{1}{N}\sum_{i=1}^N
    X_i||_2>\sqrt{\frac{2\sigma^2t}{N}}+\frac{bt}{3N}\right]\leq k\cdot t(e^t-t-1)^{-1}.$$
\end{theorem}

We will then make use of this theorem to derive our new concentration bounds. 
Section~\ref{s:standard} deals with the standard case, Section~\ref{s:prefix} with the prefix case and Section~\ref{s:factor} with the factor case.

\subsection{Concentration Bound for the Hankel Matrix $H_p^{U,V}$} \label{s:standard}
Let $p$ be a rational stochastic language over $\Sigma^*$, let $S$ be
a sample independently drawn according to $p$, and let $U, V\subseteq
\Sigma^*$. In this section, we compute a bound on $||H_S^{U,V}-H_p^{U,V}||_2$ which is independent from the
sizes of $U$ and $V$ and holds in particular when $U=V=\Sigma^*$.

Let $\xi$ be a random variable distributed according
to $p$, let $Z(\xi)=H_{\xi}^{U,V}-H_{p}^{U,V}$ be the random matrix defined by $Z_{u,v}={\mathbf
  1}_{\xi=uv}-p(uv)$ and let $X$ be the dilation of $Z$. 

Clearly,
$\Esp(X)=0$.  In order to apply Theorem~\ref{conc2}, it is necessary to compute the
parameters $b,\sigma$ and $k$.  We first prove a technical lemma that
will provide a bound on $\Esp(X^2)$.

\begin{lemma}For any $u,u'\in U$, $v,v'\in
  V$, $$|\Esp(Z_{uv}Z_{u'v})|\leq p(u'v)
\textrm{ and }|\Esp(Z_{uv}Z_{uv'})|\leq p(uv').$$
\end{lemma}
\begin{proof}
  \begin{align*}
    \Esp(Z_{uv}Z_{u'v})&=\Esp({\mathbf 1}_{\xi=uv}{\mathbf
      1}_{\xi=u'v})-p(uv)p(u'v)\\
&=\sum_{w\in \Sigma^*}p(w) {\mathbf 1}_{w=uv}{\mathbf
      1}_{w=u'v}-p(uv)p(u'v)\\
&=p(u'v)[{\mathbf 1}_{u=u'}-p(uv)]
  \end{align*}
and $$|\Esp(Z_{uv}Z_{u'v})|\leq p(u'v).$$
The second inequality is proved in a similar way. 
\end{proof}

Next
lemma provides parameters $b,\sigma$ and $k$ needed to apply Theorem~\ref{conc2}.
\begin{lemma}\label{param1}
$||X||_2\leq 2$, $\Esp(Tr(X^2))\leq 2S_p^{(2)}$ and $||\Esp(X^2)||_2\leq S_p^{(2)}.$
\end{lemma}
\begin{proof}

1. $\forall u\in U$, $\sum_{v\in V} |Z_{u,v}|=\sum_{v\in V} |{\mathbf
  1}_{\xi=uv}-p(uv)|\leq 1+ p(u\Sigma^*)\leq 2$. Therefore, $||Z||_{\infty}\leq 2$. In
a similar way, it can be shown that $||Z||_1\leq
2$. Hence, $$ ||X||_2=||Z||_2\leq \sqrt{||Z||_{\infty}||Z||_{1}}\leq
2.$$

2. For all $(u,u')\in U^2$~: $
  ZZ^T[u,u']=\sum_{v\in
    V}Z_{u,v}Z_{u',v}$. 

Therefore, 
\begin{align*}
  \Esp(Tr(ZZ^T)) &=\Esp(\sum_{u\in U}ZZ^T[u,u])\\
&=\Esp(\sum_{u\in U,v\in V}Z_{u,v}Z_{u,v})\\
&\leq\sum_{u\in U,v\in V}\Esp(Z_{u,v}Z_{u,v})\\
&\leq S_p^{(2)}.
\end{align*}
In a similar way, it can be proved that $\Esp(Tr(Z^TZ))\leq
S_p^{(2)}$ and therefore, $\Esp(Tr(X^2))\leq 2 S_p^{(2)}.$

3. For any $u\in U$, 
\begin{align*}
  \sum_{u'\in U}|\Esp(ZZ^T[u,u'])|&\leq \sum_{u'\in U,v\in
    V}|\Esp(Z_{uv}Z_{u'v})|\\
&\leq \sum_{u'\in U,v\in
    V} p(u'v)\\
  &\leq S_p^{(2)}.
\end{align*}
Hence, $||ZZ^T||_{\infty}\leq S_p^{(2)}.$ It can be proved, in a
similar way, that $||Z^TZ||_{\infty}\leq S_p^{(2)}$, $||ZZ^T||_{1}\leq
S_p^{(2)}$ and $||Z^TZ||_{1}\leq S_p^{(2)}$. Therefore, $||X^2||_2 \leq S_p^{(2)}.$
\end{proof}

We can now prove the main theorem of this section: 

\begin{theorem}\label{th:hankel}
Let $p$ be a rational stochastic language and let $S$ be a sample of  $N$ strings drawn i.i.d. from 
$p$. For all $t>0$, \small
$$Pr\left[||H_S^{U,V}-H_p^{U,V}||_2>\sqrt{\frac{2 S_p^{(2)} t}{N}}+\frac{2t}{3N}\right]\leq
  2t(e^t-t-1)^{-1}.$$ \normalsize
\end{theorem}

\begin{proof}
 Let $\xi_1, \ldots, \xi_N$ be $N$ independent
copies of $\xi$, let $Z_i=Z(\xi_i)$ and let $X_i$ be the dilation
of $Z_i$ for $i=1,\ldots,N$.
Lemma~\ref{param1} shows that the 4 conditions of Theorem~\ref{conc2} are fulfilled
with $b=2, \sigma^2=S_p^{(2)}\textrm{ and }k=2.$
\end{proof}

This bound is independent from $U$ and $V$. It can be noticed that the
proof also provides a dimension dependent bound by replacing
$S_p^{(2)}$ with $\sum_{(u,v)\in U\times V}p(uv)$, which may result in
a significative improvement if $U$ or $V$ are small.

%%% Local Variables: 
%%% mode: latex
%%% TeX-master: "ICML2014"
%%% End: 

\vspace{-1mm}\subsection{Bound for the prefix Hankel Matrix $H_{\overline{p}}^{U,V}$}
\label{s:prefix}

The random matrix $\overline{Z}(\xi)=\overline{H}_{\xi}^{U,V}-H_{\overline{p}}^{U,V}$ is defined
by $\overline{Z}_{u,v}={\mathbf 1}_{uv\in Pref(\xi)}
-\overline{p}(uv).$ It can easily be shown that $||\overline{Z}||_2$
may be unbounded if $U$ or $V$ are
unbounded: $||\overline{Z}||_2=\Omega(|\xi|^{1/2})$.
Hence, Theorem~\ref{conc2}
cannot be directly applied, which suggests that the
concentration of  $\overline{Z}$ around its mean
could be far weaker than the concentration of $Z$.  

For any $\eta\in[0,1]$, we define a smoothed variant of $\overline{p}$
by $$\overline{p}_{\eta}(u) =\sum_{x\in \Sigma^*}\eta^{|x|}p(ux)=\sum_{n\geq 0}\eta^np(u\Sigma^n).$$

Note that $\overline{p}_1=\overline{p}$, $\overline{p}_0=p$ and
that $p(u)\leq \overline{p}_{\eta}(u)\leq \overline{p}(u)$ for any
string $u$. Therefore,
the functions $\overline{p}_{\eta}$ are natural
intermediates between $p$ and  $\overline{p}$. Moreover, when $p$ is rational, each
$\overline{p}_{\eta}$ is also rational.

\begin{proposition}
Let $p$ be a rational stochastic language and let $\langle I,
(M_x)_{x\in \Sigma},T\rangle$ be a minimal linear
representation of $p$. Let
$\overline{T}_{\eta}=(I_d-\eta M_{\Sigma})^{-1}T$. Then,
$\overline{p}_{\eta}$ is rational and $\langle I,
(M_x)_{x\in \Sigma},\overline{T}_{\eta}\rangle$  is a linear
representation of $\overline{p}_{\eta}$. 
\end{proposition} 
\begin{proof}
For any string $u$, $\overline{p}_{\eta}(u)=\sum_{n\geq
  0}I^TM_u \eta^nM_{\Sigma}^nT=I^TM_u (\sum_{n\geq
  0}\eta^nM_{\Sigma}^n)T=I^TM_u \overline{T}_{\eta}$.
\end{proof}

Note that $T$ can be computed from $\overline{T}_{\eta}$ when $\eta$ and 
$M_{\Sigma}$ are known and therefore, it is a consistent learning
strategy to learn $\overline{p}_{\eta}$ from
the data, for some $\eta$,  and next, to derive $p$. 

For any $0\leq \eta \leq 1$, let $\overline{Z}_{\eta}(\xi)$ be the random
matrix defined by
\begin{align*}
  \overline{Z}_{\eta}[u,v]&=\sum_{x\in \Sigma^*}\eta^{|x|}{\mathbf
    1}_{\xi=uvx} -\overline{p}_{\eta}(uv)\\
&=\sum_{x\in
    \Sigma^*}\eta^{|x|}({\mathbf 1}_{\xi=uvx} -p(uvx)).
\end{align*}
for any $(u,v)\in U\times V$. 
It is clear that $\Esp(\overline{Z}_{\eta})=0$ and we show
below that $||\overline{Z}_{\eta}||_2$ is bounded if $\eta<1$. 

The moments $S_{\overline{p}_{\eta}}^{(k)}$ can naturally be
associated with $\overline{p}_{\eta} $. For any $0\leq \eta\leq 1$ and
any $k\geq 1$, let $$S_{\overline{p}_{\eta}}^{(k)}=\sum_{u_1u_2\ldots
  u_k\in \Sigma^*}\overline{p}_{\eta} (u_1u_2\ldots u_k).$$ We have
$S_{\overline{p}_{\eta}}^{(k)}= I^T(I_d-M_{\Sigma})^{-k}(I_d-\eta
M_{\Sigma})^{-1}T$ and it is clear that $S_{\overline{p}_{0}}^{(k)}=
S_p^{(k)}$ and $S_{\overline{p}_{1}}^{(k)}= S_p^{(k+1)}.$

\begin{lemma}
$$||\overline{Z}_{\eta}||_2\leq \frac{1}{1-\eta}+S_{\overline{p}_{\eta}}^{(1)}.$$
\end{lemma}
\begin{proof}
Indeed, let $u\in U$. 
\begin{align*}
  \sum_{v\in V}|\overline{Z}_{\eta}[u,v]|&\leq \sum_{v,x\in \Sigma^*}\eta^{|x|}{\mathbf
    1}_{\xi=uvx} + \sum_{v,x\in \Sigma^*}\eta^{|x|}p(uvx)\\
&\leq (1+\eta + \ldots + \eta^{|\xi|-|u|})+S_{\overline{p}_{\eta}}^{(1)}\\
&\leq \frac{1}{1-\eta}+S_{\overline{p}_{\eta}}^{(1)}.
\end{align*}
Hence, $||\overline{Z}_{\eta}||_{\infty}\leq
\frac{1}{1-\eta}+S_{\overline{p}_{\eta}}^{(1)}$. Similarly, $||\overline{Z}_{\eta}||_{1}\leq
\frac{1}{1-\eta}+S_{\overline{p}_{\eta}}^{(1)}$, which completes the
proof.
\end{proof}

When $U$ and $V$ are bounded, let $l$ be the maximal length of a
string in $U\cup V$. It can easily be shown that
$||\overline{Z}_{\eta}||_2\leq l+1+S_{\overline{p}_{\eta}}^{(1)}$ and
therefore, in that case,
\begin{equation}
||\overline{Z}_{\eta}||_2\leq Min(l+1,
\frac{1}{1-\eta})+S_{\overline{p}_{\eta}}^{(1)}\label{eq:2}
\end{equation}
which holds even if $\eta=1$.

\begin{lemma}
$|\Esp(\overline{Z}_{\eta}[u,v]\overline{Z}_{\eta}[u',v])|\leq
\overline{p}_{\eta}(u'v)$, for any $u,u',v\in \Sigma^*$.
\end{lemma}

\begin{proof}
We have \small $\Esp(({\mathbf
  1}_{\xi=w}-p(w)) ({\mathbf
  1}_{\xi=w'}-p(w')))=\Esp({\mathbf
  1}_{\xi=w}{\mathbf
  1}_{\xi=w'})-p(w)p(w')$. \normalsize Therefore,  

\small

$\Esp(\overline{Z}_{\eta}[u,v]\overline{Z}_{\eta}[u',v])$
\begin{align*}
&=
  \sum_{x,x'}\eta^{|xx'|}[\Esp({\mathbf 1}_{\xi=uvx}{\mathbf 1}_{\xi=u'vx'})-p(u'vx')p(uvx)]\\
&=
  \sum_{x,x',w}\eta^{|xx'|}p(w) {\mathbf
    1}_{w=u'vx'}[{\mathbf 1}_{w=uvx}-p(uvx)]\\
&=
  \sum_{x,x'}\eta^{|xx'|}p(u'vx')[{\mathbf 1}_{u'vx'=uvx}-p(uvx)]\\
&=  \sum_{x'}\eta^{|x'|}p(u'vx')[\sum_x\eta^{|x|}({\mathbf 1}_{u'vx'=uvx}-p(uvx))]
\end{align*}
and
$$|\Esp(\overline{Z}_{\eta}[u,v]\overline{Z}_{\eta}[u',v])|\leq \sum_{x'}\eta^{|x'|}p(u'vx')=\overline{p}_{\eta}(u'v)$$
 since
$$-1\leq -\overline{p}_{\eta}(uv)\leq \sum_x \eta^{|x|}({\mathbf
  1}_{u'vx'=uvx}-p(uvx))\leq 1$$
i.e.\
$$|\sum_x \eta^{|x|}({\mathbf
  1}_{u'vx'=uvx}-p(uvx))|\leq 1.$$\normalsize
\end{proof}

\begin{lemma}
$$||\Esp(\overline{Z}_{\eta}\ \overline{Z}_{\eta}^T)||_2\leq
S_{\overline{p}_{\eta}}^{(2)}\textrm{ and }Tr(\Esp(\overline{Z}_{\eta}\ \overline{Z}_{\eta}^T))\leq
S_{\overline{p}_{\eta}}^{(2)}.$$$$||\Esp(\overline{Z}_{\eta}^T\overline{Z}_{\eta})||_2\leq
S_{\overline{p}_{\eta}}^{(2)}\textrm{ and }Tr(\Esp(\overline{Z}_{\eta}^T\overline{Z}_{\eta}))\leq
S_{\overline{p}_{\eta}}^{(2)}.$$
\end{lemma}
\begin{proof}
Indeed,
\begin{align*}
||\Esp(\overline{Z}_{\eta}\overline{Z}_{\eta}^T)||_{\infty}&\leq
Max_u\sum_{u',v}|\Esp(\overline{Z}_{\eta}[u,v]\overline{Z}_{\eta}[u',v])|\\
&\leq \sum_{u',v,x'}\eta^{|x'|}p(u'vx') \leq S_{\overline{p}_{\eta}}^{(2)}.
\end{align*}
In the same way,
$$Tr(\Esp(\overline{Z}_{\eta}\overline{Z}_{\eta}^T))=\sum_{u,v}\Esp(\overline{Z}_{\eta}[u,v]\overline{Z}_{\eta}[u,v])\leq
S_{\overline{p}_{\eta}}^{(2)}.$$
Similar computations provide all the inequalities. 
\end{proof}

Therefore, we can apply the Theorem~\ref{conc2} with $b=\frac{1}{1-\eta}+S_{\overline{p}_{\eta}}^{(1)},
\sigma^2=S_{\overline{p}_{\eta}}^{(2)}$ and $k=2$.

\begin{theorem}\label{th:hankelPrefix}
Let $p$ be a rational stochastic language, let $S$ be a sample of  $N$ strings drawn i.i.d. from 
$p$ and let $0\leq\eta<1$. For all $t>0$, 
$$Pr\left[||\overline{H}_{\eta,S}^{U,V}-H_{\overline{p}_{\eta}}^{U,V}||_2>\sqrt{\frac{2
      S_{\overline{p}_{\eta}}^{(2)}
      t}{N}}+\frac{t}{3N}\left[\frac{1}{1-\eta}+S_{\overline{p}_{\eta}}^{(1)}\right]\right]$$
$$\leq
  2t(e^t-t-1)^{-1}.$$
\end{theorem}

Remark that when $\eta=0$ we find back the concentration bound of
Theorem~\ref{th:hankel}, and that Inequality~\ref{eq:2} provides a
bound when $\eta=1$. 
%%% Local Variables: 
%%% mode: latex
%%% TeX-master: "ICML2014"
%%% End: 

\vspace{-1mm}\subsection{Bound for the factor Hankel Matrix $H_{\widehat{p}^{U,V}}$}
\label{s:factor}

The random matrix $\widehat{Z}(\xi)=\widehat{H}_{\xi}^{U,V}-H_{\widehat{p}^{U,V}}$ is defined
by $$\widehat{Z}_{u,v}=\sum_{x,y\in
\Sigma^*}{\mathbf 1}_{\xi=xuvy}-\widehat{p}(uv).$$ 
$||\widehat{Z}||_2$ is generally unbounded. Moreover, unlike the
prefix case,  $||\widehat{Z}||_2$ can be unbounded even if $U$ and $V$
are finite. Hence, the Theorem~\ref{conc2}
cannot be directly applied either.  

We can also define smoothed variants of $\widehat{p}$
by $$\widehat{p}_{\eta}(u) =\sum_{x,y\in \Sigma^*}\eta^{|xy|}p(xuy)=\sum_{m,n\geq 0}\eta^{m+n}p(\Sigma^mu\Sigma^n)$$
which have properties similar to  functions
$\overline{p}_{\eta}$:
\begin{itemize}
\item $p \leq \widehat{p}_{\eta}\leq \widehat{p}$, $\widehat{p}_1=\widehat{p}$ and $\widehat{p}_0=p$,
\item if $\langle I,
(M_x)_{x\in \Sigma},T\rangle$ be a minimal linear
representation of $p$ then $\langle \widehat{I}_{\eta},
(M_x)_{x\in \Sigma},\overline{T}_{\eta}\rangle$, where $\widehat{I}_{\eta}=(I_d-\eta M_{\Sigma}^T)^{-1}I$,  is a linear
representation of $\hat{p}_{\eta}$. 
\end{itemize}

However, proofs of the previous Section cannot be directly
extended to $\widehat{p}_{\eta}$ because $\overline{p}$ is bounded by
1, a property which is often used in the proofs, while $\widehat{p}$
is not. Next lemma provides a tool which allows to bypass this difficulty.  

\begin{lemma}Let $0<\eta\leq 1$. For any integer
  $n$, $(n+1)\eta^n\leq K_{\eta}$ where   $$K_{\eta}=\left\{
  \begin{array}{cl}
    1 &\textrm{ if }\eta\leq e^{-1} \\
(-e\eta\ln \eta)^{-1}& \textrm{ otherwise.}\\
  \end{array}
\right.$$
\end{lemma}
\begin{proof}
  Let $f(x)=(x+1)\eta^x$. We have $f'(x)=\eta^x(1+(x+1)\ln \eta)$ and
  $f$ takes its maximum for $x_M=-1-1/\ln \eta$, which is positive if
  and only if $\eta> 1/e$. We have $f(x_M) =(-e\eta\ln \eta)^{-1}$.
\end{proof}

\begin{lemma}
Let $w,u\in \Sigma^*$.  Then, 
$$\sum_{x,y\in \Sigma^*}\eta^{|xy|}{\mathbf
  1}_{w=xuy}\leq K_{\eta}\textrm{ and }\widehat{p}(u)\leq K_{\eta}p(\Sigma^*u\Sigma^*).$$
\end{lemma}

\begin{proof}
Indeed, if $w=xuy$, then $|xy|=|w|-|u|$ and $u$ appears at most
$|w|-|u|+1$ times as a factor of $w$. 
\begin{align*}
  \widehat{p}(u)&=\sum_{x,y\in \Sigma^*}\eta^{|xy|}p(xuy)\\
&=\sum_{w\in \Sigma^*u\Sigma^*}p(w)\sum_{x,y\in \Sigma^*}\eta^{|xy|}{\mathbf
  1}_{w=xuvy}\\
&\leq K_{\eta}p(\Sigma^*u\Sigma^*).
\end{align*}
\end{proof}

For $\eta \in [0,1]$, let $\widehat{Z}_{\eta}(\xi)$ be the random
matrix defined by
\begin{align*}
  \widehat{Z}_{\eta}[u,v]&=\sum_{x,y\in \Sigma^*}\eta^{|xy|}{\mathbf
    1}_{\xi=xuvy} -\widehat{p}_{\eta}(uv)\\
&=\sum_{x,y\in
    \Sigma^*}\eta^{|xy|}({\mathbf 1}_{\xi=xuvy} -p(xuvy)).
\end{align*}
and, for any $k\geq 0$,  let
$$S_{\widehat{p}_{\eta}}^{(k)}=\sum_{u_1u_2\ldots u_k\in
  \Sigma^*}\widehat{p}_{\eta}(u_1u_2\ldots u_k).$$
It can easily be shown that $\Esp(\widehat{Z}_{\eta})=0$, $S_{\widehat{p}_{\eta}}^{(k)}=
I^T(I_d-\eta M_{\Sigma})^{-1}(I_d-M_{\Sigma})^{-k}(I_d-\eta M_{\Sigma})^{-1}T$, $S_{\widehat{p}_{0}}^{(k)}=
S_p^{(k)}$ and $S_{\widehat{p}_{1}}^{(k)}=
S_p^{(k+2)}.$

It can be shown that $||\widehat{Z}_{\eta}||_2$ is bounded if $\eta<1$. 

\begin{lemma}\ifthenelse{\boolean{vl}}{}{\cite{Denis2013}}$$||\widehat{Z}_{\eta}||_2\leq (1-\eta)^{-2}+S_{\widehat{p}_{\eta}}^{(1)}.$$
\end{lemma}
\ifthenelse{\boolean{vl}}{
\begin{proof}
Indeed, for all $u$,
\begin{align*}
  \sum_{v\in V}|\widehat{Z}_{\eta}[u,v]|&\leq \sum_{v,x,y\in \Sigma^*}\eta^{|xy|}{\mathbf
    1}_{\xi=xuvy} + \hat{p}_{\eta}(uv)\\
&\leq (1+\eta + \ldots + \eta^{|\xi|-|u|})^2 +S_{\hat{p}_{\eta}}^{(1)}\\
&\leq \frac{1}{(1-\eta)^2}+S_{\overline{p}_{\eta}}^{(1)}.
\end{align*}
Hence, $||\widehat{Z}_{\eta}||_{\infty}\leq
\frac{1}{(1-\eta)^2}+S_{\hat{p}_{\eta}}^{(1)}$. Similarly, $||\overline{Z}_{\eta}||_{1}\leq\frac{1}{(1-\eta)^2}+S_{\hat{p}_{\eta}}^{(1)}$, which completes the
proof.
\end{proof}
}
{}

\ifthenelse{\boolean{vl}}{
\begin{lemma}For any $u,u',v\in \Sigma^*$, 
$|\Esp(\widehat{Z}_{\eta}[u,v]\widehat{Z}_{\eta}[u',v])|\leq K_{{\eta}}\sum_{x'y'}{\eta}^{|x'y'|}p(x'u'vy')$.
\end{lemma}

\begin{proof}
We have 

$\Esp(\widehat{Z}_{\eta}[u,v]\widehat{Z}_{\eta}[u',v])=$
\begin{align*}
  &\sum_{x,x',y,y'}\eta^{|xx'yy'|}[\Esp({\mathbf 1}_{\xi=xuvy}{\mathbf
    1}_{\xi=x'u'vy'}) \\
&-p(x'u'vy')p(xuvy)].
\end{align*}

We remark that 

$\Esp({\mathbf 1}_{\xi=xuvy}{\mathbf
  1}_{\xi=x'u'vy'})-p(x'u'vy')p(xuvy)$
\begin{align*}
  &=\sum_wp(w) {\mathbf 1}_{w=x'u'vy'}({\mathbf 1}_{w=xuvy}-p(xuvy)),
\end{align*}

and therefore,  
$\Esp(\widehat{Z}_{\eta}[u,v]\widehat{Z}_{\eta}[u',v])=$
\begin{align*}
&\sum_{x',y',w}\eta^{|x'y'|}p(w){\mathbf
  1}_{w=x'u'vy'}(\sum_{x,y} \eta^{|xy|}({\mathbf
  1}_{w=xuvy}-
p(xuvy))).
\end{align*}

Moreover, $|\sum_{xy} \eta^{|xy|}({\mathbf
  1}_{w=xuvy}-p(xuvy))|\leq K_{\eta}$. 
\end{proof}
}
{}
\ifthenelse{\boolean{vl}}{

\begin{lemma}
$$||\Esp(\widehat{Z}\widehat{Z}^T)||_2\leq
K_{\eta} S_{\widehat{p}_{\eta}}^{(2)}\textrm{ and }Tr(\Esp(\widehat{Z}\widehat{Z}^T))\leq
K_{\eta} S_{\widehat{p}_{\eta}}^{(2)}.$$
\end{lemma}
\begin{proof}
We have $$||\Esp(\widehat{Z}\widehat{Z}^T)||_{\infty}\leq
Sup_u\sum_{u',v}|\Esp(\widehat{Z}_{\eta}[u,v]\widehat{Z}_{\eta}[u',v])|$$
Then from previous lemma: \\
$\sum_{u',v}|\Esp(\widehat{Z}_{\eta}[u,v]\widehat{Z}_{\eta}[u',v])|\leq
K_{\eta}S_{\widehat{p}_{\eta}}^{(2)}$ for any $u\in \Sigma^*$.  
Finally, \\
$$Tr(\Esp(\widehat{Z}\widehat{Z}^T))=\sum_{u,v}\Esp(\widehat{Z}_{\eta}[u,v]\widehat{Z}_{\eta}[u,v])\leq K_{\eta}S_{\widehat{p}_{\eta}}^{(2)}.$$
\end{proof}

Similar proof gives \begin{lemma}
$$||\Esp(\widehat{Z}^T\widehat{Z})||_2\leq K_{\eta}
S_{\widehat{p}_{\eta}}^{(2)}\textrm{ and }Tr(\Esp(\widehat{Z}^T\widehat{Z}))\leq
K_{\eta} S_{\widehat{p}_{\eta}}^{(2)}.$$
\end{lemma}

}
{}

Eventually, we can apply the Theorem~\ref{conc2} with $b=(1-\eta)^{-2}+S_{\widehat{p}_{\eta}}^{(1)},
\sigma^2=K_{\eta}S_{\widehat{p}_{\eta}}^{(2)}$ and $k=2$\ifthenelse{\boolean{vl}}{}{~\cite{Denis2013}}.

\begin{theorem}\label{th:hankelFactor}
Let $p$ be a rational stochastic language, let $S$ be a sample of  $N$ strings drawn i.i.d. from 
$p$ and let $0\leq\eta<1$. For all $t>0$, 
\small
$$\!Pr\left[||\widehat{H}_{\eta,S}^{U,V}-H_{\widehat{p}_{\eta}}^{U,V}||_2>\!\sqrt{\frac{2 K_{\eta}
      S_{\widehat{p}_{\eta}}^{(2)}
      t}{N}}+\frac{t}{3N} \left[\frac{1}{(1-\eta)^2}+S_{\widehat{p}_{\eta}}^{(1)} \right]\!\right]$$
$$\leq
  2t(e^t-t-1)^{-1}.$$
\normalsize
\end{theorem}

Remark that when $\eta=0$ we find back the concentration bound of
Theorem~\ref{th:hankel}. 
We provide experimental evaluation of the proposed bounds in the next
Section. 
%%% Local Variables: 
%%% mode: latex
%%% TeX-master: "ICML2014"
%%% End: 

\section{Experiments}
\label{s:expe}
The proposed bounds are evaluated on the benchmark of PAutomaC~
\cite{VerwerEH12}
which provides samples of strings generated from several probabilistic
automata, designed to evaluate probabilistic automata learning. Eleven problems have been selected from that benchmark for which sparsity
of the Hankel matrices makes the use
of standard SVD algorithms available from \texttt{NumPy} or
\texttt{SciPy}  possible. 
\ifthenelse{\boolean{vl}}{Table \ref{tab:pb_pautomac} provides some
  information about the selected problems.

\begin{table*}[htb]
\vspace{-1mm}  
  \caption{ Properties of the 11 selected problems. Target models are
    of different types: non deterministic
    probabilistic finite automata (PFA), deterministic PFA (DPFA) and
    hidden Markov models (HMM). The size of the Hankel matrices matrices
   is expressed in billions, where g stands for $1\times 10^9$. The sparsity
   is indicated as the percentage of non zero entries in the matrix.}
\vspace{3mm}  
\begin{scriptsize}
  \centering
  \begin{tabular}{|l|l|l|l|l|l|l|l|l|l|l|l|}
    \hline
    \textbf{Problem number} & \textbf{3} & \textbf{4} &
    \textbf{7}& \textbf{15} & \textbf{25} &\textbf{29}&
    \textbf{31}& \textbf{38} & \textbf{39} &\textbf{40}& \textbf{42} \\ \hline
 \textbf{Alphabet size}& 4& 4& 13&14 &10 & 6&5 &10& 14&14&9   \\
$S^{(2)}_p$  & 8.23 &  6.25 &  6.52 &  13.40 &  10.65 &  6.35 &  6.97 &  8.09 &  8.82 &  9.74 &  7.39\\
$S^{(3)}_p$  & 57.84 &  31.06 &  29.61 &  160.92 &  93.34 &  38.11 &
43.53 &  65.87 &  90.81 &  111.84 &  62.11 \\
\hline
 \textbf{Average string length}&7.219 &5.259 &5.523 &12.461 & 9.723& 5.287&6.001 &7.177&7.736 & 8.716&6.350  \\
 \textbf{Max. string length}  & 67& 55& 36& 110& 90& 59& 59&84& 106& 106& 70  \\
\textbf{Size sample $N$}&20000 &100000 &20000 &20000 &20000 &20000&20000 &20000&20000 &20000& 20000  \\
\textbf{Type of target model} & PFA& PFA&DPFA &PFA &HMM &PFA &PFA &HMM&PFA& DPFA&DPFA  \\
\textbf{Nb of states in the target} & 25& 12& 12&26 &40 & 36&12 &14& 6&65& 6\\\hline
\textbf{Size $H_S^{U,V}$ standard} &1.9g &0.5g &0.17g &27g &13g&0.4g &1.4g  &8g&7.7g &15g&3.4g  \\
\textbf{Sparsity}& .0053\% &.0185\% &.0212\% &.0009\% &.0015\% &.0116\%&.0061\% &.0018\%& .0019\%&.0011\%&.0033\%  \\\hline
\textbf{Size of $\overline{H}_S^{U,V}$ prefix} &2.5g &1.8g &0.7g &291g &99g & 2.4g&7.6g& 60g&75g&165g &25g  \\
\textbf{Sparsity}&.0058\% &.0191\% &.0208\% &.0001\% &.0016\% &.0122\% &.0066\% &.0019\%&.0020\% &.0012\%&.0035\%   \\\hline
\textbf{Size of $\hat{H}_S^{U,V}$ factor} &73g &6.4g &3g &3363g &797g & 15.7g& 44g &460g& 761g & 1925g&202g  \\
\textbf{Sparsity}&.0058\%& .0197\% &.0199\% &.0001\% &.0016\% &.0115\% &.0069\% &.0020\%&.0020\% &.0012\%&.0036\%  \\\hline
  \end{tabular}
\end{scriptsize}
  \label{tab:pb_pautomac}
\end{table*}
%%% Local Variables: 
%%% mode: latex
%%% TeX-master: "ICML2014"
%%% End: 

}{
The size $N$ of the samples is $20000$ except for the
problem $4$ where $N=100000$. Table~\ref{fig:table_Sp} shows some
target properties of the
selected problems: the size of the alphabets and the exact values of $S_p^{(k)}$
computed for the different targets $p$.
\input{Table_Sp_2}
}
Figure~\ref{fig:Sp} shows the
typical behavior of $S_{\overline{p}_{\eta }}^{(1)}$ and
 $S_{\widehat{p}_{\eta }}^{(1)}$, similar for all the problems.
\begin{figure}[H]
 \centering
 \includegraphics[width=8.6cm,height=4.5cm]{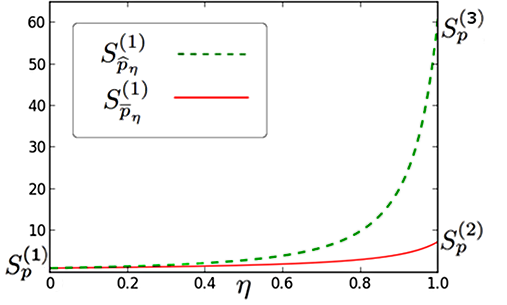}
\caption{Behavior of $S_{\overline{p}_{\eta}}^{(1)}$ and $S_{\widehat{p}_{\eta}}^{(1)}$
  for $\eta \in [0;1]$.}
\label{fig:Sp}
\end{figure}
For each problem, the exact value of $||H^{U,V}_S
-H^{U,V}_p||_2$ is computed for sets $U$ and $V$ of the form $\Sigma^{\leq l}$, trying to
maximize $l$ according to our computing resources. 
It is compared to the bounds provided by
Theorem~\ref{th:hankel} and Equation~\eqref{eq:1}, with 
$\delta=0.05$ (Table~\ref{fig:table_classical}). 
The optimized bound ("opt.''), refers to the case where
$\sigma^2$ has been calculated over $U \times V$ rather than $\Sigma^* \times \Sigma^*$
(see the remark at the end of Section~\ref{s:standard}). Tables~\ref{fig:table_prefix} and~\ref{fig:table_factor} show analog
comparisons for the prefix and the factor cases with different values
of $\eta$. Similar results have been obtained for
all the problems of PautomaC.
\begin{table*}[htb]
 \vspace{-1mm}  \caption{Concentration values from various bounds for $||H^{U,V}
    _S-H^{U,V}_p||_2$ for $U=V=\Sigma^{\leq l}$.} \vspace{3mm}
{\footnotesize
  \centering
  \begin{center}
    \begin{tabular}{|c|c|c|c|c|c|c|c|c|c|c|c|}
    %\begin{tabular}{|c|c|ccccccccccc|}
      \hline Problem number  &  3     &  4     &  7     &  15    &  25  &  29
      &    31     &  38     &  39     &  40     &  42        \\

      \hline 
$l$ & 8 &  9 &  8 &  5 &  5 &  9 &  7 &  4 &  6 &  4 &  7 \\
      % \hline $|U|=|V|$ & 2167 &  2381 &  2382 &  2145 &  2988 & 2115 &  2200 &  1036 &  2645 &  775 &  3056    \\
      
     \hline
$||H^{U,V}_S - H_p^{U,V}||_2$ &0.0052 &  0.0030 &  0.0064 &  0.0037 &  0.0033 &  0.0045 &  0.0051 &  0.0058 &  0.0049 &  0.0037 &  0.0054\\
      
     % \hline 
Eq.~\eqref{eq:1} &  0.1910 &  0.0857 &  0.1917 &  0.1909 &
      0.1935 &  0.1908 &  0.1911 &  0.1852 &  0.1925 &  0.1829 &
      0.1936    \\

%      \hline 
\textbf{Th.~\ref{th:hankel}
        (dim. free)}  & 0.0669 &  0.0260 &  0.0595 &  0.0853 &
      0.0761 &  0.0588 &  0.0615 &  0.0663 &  0.0692 &  0.0728 &
      0.0634 \\
      
%      \hline 
\textbf{Th.~\ref{th:hankel} (opt. $\boldmath{U,V}$) } & 0.0475 &  0.0228 &  0.0527 &
      0.0284 &  0.0323 &  0.0472 &  0.0437 &  0.0275 &  0.0325 &
      0.0243 &  0.0378  \\
          
      \hline 
    \end{tabular}
  \end{center}
}
\label{fig:table_classical}
\end{table*}
%%% Local Variables: 
%%% mode: latex
%%% TeX-master: "ICML2014"
%%% End: 

\begin{table*}[!htb]
 \vspace{-1mm}\caption{Concentration values from various bounds for
    $||\overline{H}^{U,V}_S-H ^{U,V}_{\overline{p},\eta}||_2$ (prefix case)  for $U=V=\Sigma^{\leq l}$.}\vspace{3mm}
{\footnotesize
  \centering
 \hspace{-4.1mm} \begin{tabular}{|c|c|c|c|c|c|c|c|c|c|c|c|c|c|}
% \hspace{-4.1mm} \begin{tabular}{|c|c|ccccccccccc|}
    \hline\multicolumn{2}{|c|}{Problem number} & 3     &  4     &  7     &  15    &  25  &  29
      &    31     &  38     &  39     &  40     &  42        \\

      \hline \multicolumn{2}{|c|}{$l$} & 8 &  9 &  8 &  5 &  5 &  9 &  7 &  4 &  6 &  4 &  7 \\    
      %\hline \multicolumn{2}{|c|}{$|U|=|V|$}  & 2167 &  2381 &  2382 & 2145 &  2988 &  2115 &  2200 &  1036 &  2645 &  775 &  3056   \\
      
%      \hline
\hline\hspace{-1.4mm}$||\overline{H}^{U,V}_S -
      H_{\overline{p},\eta}^{U,V}||_2$\hspace{-1.4mm}& & 0.0067 &
      0.0035  & 0.0085 & 0.0043 & 0.0041  & 0.0055 & 0.0073 &0.0059  & 0.0061 & 0.0044 & 0.0062 \\
    
%      \cline{1-1}\cline{3-13} 
Eq.~\eqref{eq:1} &\hspace{-1.6mm}$\eta=\frac{1}{2}$\hspace{-1.4mm} & 0.7463 &  0.3326 &  0.7545 &  0.7515 &
      0.7626 &  0.7250 &  0.7369 &  0.7051 &  0.7068 &  0.6753 &
      0.7146 \\

    %  \cline{1-1}\cline{3-13} 
\textbf{Th.~\ref{th:hankelPrefix}
        (dim. free)} & & 0.0890 &  0.0339 &  0.0777 &  0.1162 &
      0.1026 &  0.0770 &  0.0811 &  0.0884 &  0.0931 &  0.0983 &
      0.0844 \\
      
%     \cline{1-1}\cline{3-13} 
\textbf{Th.~\ref{th:hankelPrefix}
        (opt. $\boldmath{U,V}$)} & & 0.0636 &  0.0299 &  0.0697 &
      0.0398 &  0.0457 &  0.0621 &  0.0577 &  0.0366 &  0.0432 &
      0.0317 &  0.0498 \\ 

      \hline\hline\hspace{-1.4mm} $||\overline{H}^{U,V}_S -
      H_{\overline{p},\eta}^{U,V}||_2$\hspace{-1.5mm}& & 0.0141 &  0.0059 &  0.0217 &  0.0124 &  0.0145 &  0.0116 &  0.0182 &  0.0132 &  0.0135 &  0.0089 &  0.0127 \\
    
   %   \cline{1-1}\cline{3-13} 
Eq.~\eqref{eq:1} &\hspace{-0.6mm}$\eta=\small1$\hspace{-2mm}& 3.1011 &  1.3079 &  2.7839 &  3.5129 &
      3.0283 &  2.9286 &  2.6695 &  2.2395 &  2.8524 &  2.5132 &
      2.7863 \\

 %     \cline{1-1}\cline{3-13} 
\textbf{Th.~\ref{th:hankelPrefix} (dim. free)} & &  0.1784 &  0.0582 &  0.1279 &  0.2967 &
      0.2261 &  0.1450 &  0.1547 &  0.1899 &  0.2230 &  0.2472 &
      0.1846  \\

    %  \cline{1-1}\cline{3-13} 
\textbf{Th.~\ref{th:hankelPrefix}
        (opt. $\boldmath{U,V}$)} & &0.1281 &  0.0518 &  0.1166 &
      0.1062 &  0.1057 &  0.1175 &  0.1099 &  0.0778 &  0.1020 &
      0.0761 &  0.1077 \\

      \hline 
    \end{tabular}
%  \end{center}
}
\label{fig:table_prefix}
\end{table*}
%%% Local Variables: 
%%% mode: latex
%%% TeX-master: "ICML2014"
%%% End: 

\begin{table*}[!htb]
 \vspace{-1mm}  \caption{Concentration values from various bounds for
    $||\widehat{H}^{U,V}_S-H ^{U,V}_{\hat{p},\eta}||_2$ (factor case)   for $U=V=\Sigma^{\leq l}$.} \vspace{3mm}
{\footnotesize
  \centering
 \hspace{-3.2mm}\begin{tabular}{|c|c|c|c|c|c|c|c|c|c|c|c|c|}
%   \hspace{-3.2mm}\begin{tabular}{|c|c|ccccccccccc|}
    \hline \multicolumn{2}{|c|}{Problem number}  &  3     &  4     &  7     &  15    &  25  &  29
      &    31     &  38     &  39     &  40     &  42        \\

      \hline \multicolumn{2}{|c|}{$l$} & 6 &  7 &  5  &  4  &  4  & 6  &  6 & 4  &  4 & 4  & 5  \\
      % \hline \multicolumn{2}{|c|}{$|U|=|V|$} & 1348 &  1874 & 1253 & 2674 & 1110 & 2908  & 2535 &1111  & 845 & 2953  &   1566 \\
      
      \hline\hspace{-0.5mm}$||\widehat{H}^{U,V}_S - H_{\widehat{p},\eta}^{U,V}||_2$\hspace{-1mm}& & 0.0065 &
      0.0031 & 0.0071 & 0.0042 & 0.0033  & 0.0051 & 0.0072 & 0.0061 & 0.0065 &0.0047  & 0.0060 \\

   %   \cline{1-1}\cline{3-13}
 Eq.~\eqref{eq:1} &\hspace{-0.9mm}$\eta=\frac{1}{e}$\hspace{-0.7mm} & 0.9134 &  0.4107 &  0.9196 & 0.9466 &  0.9152 &
      0.9096 &  0.9219 & 0.8765 & 0.8292 & 0.8796 & 0.8565 \\

 %     \cline{1-1}\cline{3-13}
\textbf{Th.~\ref{th:hankelFactor}
        (dim. free)} & & 0.0985 &  0.0374 &  0.0858 & 0.1292 &
      0.1139 & 0.0849 & 0.0895 &0.0979 & 0.1033 & 0.1092 & 0.0934   \\
      
 %     \cline{1-1}\cline{3-13} 
\textbf{Th.~\ref{th:hankelFactor}
        (opt. $\boldmath{U,V}$)} & & 0.0601 &  0.0300 &  0.0619 & 0.0364 &0.0412 & 0.0559
      & 0.0589 &  0.0405 &  0.0356 & 0.0349 &  0.0444 \\
      \hline 
    \end{tabular}
} \vspace{-1mm}
\label{fig:table_factor}
\end{table*}
%%% Local Variables: 
%%% mode: latex
%%% TeX-master: "ICML2014"
%%% End: 

We can remark that our dimension-free bounds are significantly more
accurate than the one provided by Equation~\eqref{eq:1}. Notice that in the prefix
case, the dimension-free bound has a better behavior in the limit
case $\eta=1$ than the bound
from Eq.~\eqref{eq:1}. This is due to the fact that in
our bound, the term that bounds $||Z||_2$ appears in the $\frac{1}{N}$ term while
it appears in the $\frac{1}{\sqrt{N}}$ term in the other one.

\ifthenelse{\boolean{vl}}{\paragraph{Implication for learning}

These results show that the concentration of the empirical Hankel
matrix around its mean does not highly depend on its dimension and
they suggest that as far as computational resources permit it, the
size of the matrices should not be artificially restricted in spectral
algorithms for learning HMMs or rational stochastic languages. 

To illustrate this claim, we have performed additional experiments by
considering matrices with 3,000 columns and a variable number of rows,
from 70 to 3,000.

For each problem and each set of rows and columns, we have computed
the $r$ first right singular vectors $R$ of $H^{U,V}$ (resp. $R_S$ of
$H_S^{U,V}$), where $r$ is the rank of the target, and the distance
between the linear spaces spanned by $R$ and $R_S$. Most classical
distances are based on the principal angles $\theta_1\geq \theta_2
\geq \ldots \geq \theta_r$ between the spaces $span(R)$ and
$span(R_S)$. The largest principal angle $\theta_1$ is a harsh measure
since, even if the two spaces coincide along the last $r-1$ principal
angles, the distance between the two spaces can be large. We have
considered the following measure
\begin{equation}
d(span(R),span(R_S))=1-\frac{1}{r}\sum_{i=1}^r\cos{\theta_i}\label{eq:3}
\end{equation}
which is equal to 0 if the
spaces coincide and 1 if they are completely orthogonal, and which takes into account
all the principal angles.

The table~\ref{tab:angles1} shows the sum $\sum_{i=1}^r\cos{\theta_i}$
for each problem. The table~\ref{tab:angles2} displays the same
information but each measure is normalised by using
formula~\ref{eq:3}.

These tables show that for all problems but two, the spaces spanned by
the right singular vectors are the closest for the maximal size
Hankel matrix. They also show that these spaces remain quite distant
for 6 problems over 11. For 4 problems, the spaces are already close to each
other even for small matrices - but it can be noticed that widening
the matrix do not deteriorate the results. 

\begin{table*}[t]
\vspace{-1mm}
  \caption{Sum of the cosinus of the principal angles for each problem using matrices of dimension $|U|\times 3000$.}
\vspace{3mm}
  \centering
%\begin{scriptsize}
\begin{tabular}{|c||c|c|c|c|c|c|c|c|c|c|c|}
\hline
$|U|$&3&4&7&15&25&29&31&38&39&40&42\\
\hline
70&17.67&9.97&10.98&19.08&5.63&23.94&10.89&6.46&5.16&39.31&5.91\\
100&17.97&9.97&10.98&19.18&6.59&26.80&10.93&6.75&5.19&40.31&\textbf{5.95}\\
200&18.31&\textbf{9.98}&\textbf{11.99}&20.17&6.93&26.99&11.12&7.13&\textbf{5.82}&42.14&\textbf{5.95}\\
500&18.76&\textbf{9.98}&\textbf{11.99}&21.47&7.13&27.49&11.13&7.94&5.63&43.54&\textbf{5.95}\\
1000&18.82&\textbf{9.98}&\textbf{11.99}&21.53&7.82&27.88&11.16&8.38&5.60&44.94&\textbf{5.95}\\
2000&19.02&\textbf{9.98}&\textbf{11.99}&21.76&\textbf{7.98}&27.86&11.19&8.38&5.52&45.22&\textbf{5.95}\\
3000&\textbf{19.07}&\textbf{9.98}&\textbf{11.99}&\textbf{21.79}&7.61&\textbf{27.89}&\textbf{11.20}&\textbf{8.47}&5.48&\textbf{45.68}&\textbf{5.95}\\
\hline
rank&25&10&12&26&14&36&12&13&6&65&6\\
\hline
\end{tabular}

%\end{scriptsize}
  \label{tab:angles1}
\end{table*}

\begin{table*}[t]
\vspace{-1mm}
  \caption{Normalized distance between the principal angles for each problem using matrices of dimension $|U|\times 3000$.}
\vspace{3mm}
  \centering
%\begin{scriptsize}
\begin{tabular}{|c||c|c|c|c|c|c|c|c|c|c|c|}
\hline
$|U|$&3&4&7&15&25&29&31&38&39&40&42\\
\hline
70&0,293&0,003&0,085&0,266&0,598&0,335&0,093&0,503&0,140&0,395&0,015\\
100&0,281&0,003&0,085&0,262&0,529&0,256&0,089&0,481&0,135&0,380&0,008\\
200&0,268&\textbf{0,002}&\textbf{0,001}&0,224&0,505&0,250&0,073&0,452&\textbf{0,030}&0,352&0,008\\
500&0,250&\textbf{0,002}&\textbf{0,001}&0,174&0,491&0,236&0,072&0,389&0,062&0,330&0,008\\
1000&0,247&\textbf{0,002}&\textbf{0,001}&0,172&0,441&0,226&0,070&0,355&0,067&0,309&0,008\\
2000&0,239&\textbf{0,002}&\textbf{0,001}&0,163&\textbf{0,430}&0,226&0,068&0,355&0,080&0,304&0,008\\
3000&\textbf{0,237}&\textbf{0,002}&\textbf{0,001}&\textbf{0,162}&0,456&\textbf{0,225}&\textbf{0,067}&\textbf{0,348}&0,087&\textbf{0,297}&\textbf{0,008}\\
\hline
rank&25&10&12&26&14&36&12&13&6&65&6\\
\hline
\end{tabular}

%\end{scriptsize}
  \label{tab:angles2}
\end{table*}

% \begin{table*}[t]
%   \centering
% %\begin{scriptsize}
% \begin{tabular}{|c||c|c|c|c|c|c|c|c|c|c|c|}\hline
% $|S|$&3&4&7&15&25&29&31&38&39&40&42\\\hline
%     70 &0.00446 &0.00307 &0.01982 &0.00484 &0.00570 &0.00789 &0.00397 
% &0.00516 &0.00408 & 0.00388&0.00420\\
%    100 &0.00448 &0.00305 &0.01983 &0.00476 &0.00385 &0.00518 &0.00397 
% &0.00516 &0.00405 & 0.00383&0.00397\\
%    200 &0.00442 &0.00304 &0.00533 &0.00434 &0.00383 &0.00531 &0.00393 
% &0.00517 &0.00363 & 0.00380&0.00397\\
%    500 &0.00453 &0.00303 &0.00532 &0.00412 &0.00367 &0.00534 &0.00393 
% &0.00519 &0.00384 & 0.00379&0.00395\\
%   1000 &0.00455 &0.00302 &0.00530 &0.00412 &0.00346 &0.00513 &0.00394 
% &0.00522 &0.00387 & 0.00380&0.00396\\
%   2000 &0.00457 &0.00302 &0.00530 &0.00410 &0.00348 &0.00518 &0.00395 
% &0.00522 &0.00394 & 0.00380&0.00395\\
 %   3000 &0.00458 &0.00301 &0.00530 &0.00410 &0.00350 &0.00516 &0.00395 
% &0.00523 &0.00397 & 0.00381&0.00395\\\hline
% rank&25&10&12&26&14&36&12&13&6&65&6
% \end{tabular}

% %\end{scriptsize}
%   \caption{ }
%   \label{tab:2}
% \end{table*}
%%% Local Variables: 
%%% mode: latex
%%% TeX-master: "ICML2014"
%%% End: 
}{}

%%% Local Variables: 
%%% mode: latex
%%% TeX-master: "ICML2014"
%%% End: 

\section{Conclusion}
\label{s:conclu}
We have provided dimension-free concentration inequalities for Hankel
matrices in the context of spectral learning of rational stochastic languages. 
These bounds cover 3 cases, each one corresponding to a specific way to exploit
the strings under observation, paying attention to the strings
themselves, to their prefixes or to their factors. For the last two
cases, we introduced parametrized variants which allow a trade-off
between the rate of the concentration and the exploitation of the
information contained in data.

A consequence of these results is that there is no a priori good
reason, aside from computing resources limitations, to restrict the
size of the Hankel matrices. This suggests an immediate future work
consisting in investigating recent random techniques \cite{Halko:2011}
to compute singular values decomposition on Hankel matrices in order
to be able to deal with huge matrices. Then, a second aspect is to evaluate the impact of these
methods on the quality of the models, including an empirical
evaluation of the behavior of the standard approach and its prefix
and factor extensions, along with the influence of the parameter
$\eta$.  

Another research direction would be to link up the prefix and factor
cases to concentration bounds for sum of random tensors and to generalize the results to the
case where a fixed number $\geq 1$ of factors is considered for each string.
%%% Local Variables: 
%%% mode: latex
%%% TeX-master: "ICML2014"
%%% End: 

\section*{Acknowledments}
This work was supported by the French National Agency for Research (Lampada - ANR-09-EMER-007). 
%\bibliographystyle{icml2014}
%\bibliography{icml2014}

\end{document}